%% file: aaai24.tex
\documentclass[letterpaper]{article} 
\usepackage{aaai24}  
\usepackage{times}  
\usepackage{helvet}  
\usepackage{courier}  
\usepackage[hyphens]{url}  
\usepackage{graphicx} 
\usepackage{natbib}  
\usepackage{caption} 
\usepackage[ruled,linesnumbered, noend]{algorithm2e}
\usepackage[table,xcdraw]{xcolor}
\SetAlgoSkip{}
\SetCommentSty{emph}

\usepackage{newfloat}
\usepackage{listings}
\DeclareCaptionStyle{ruled}{labelfont=normalfont,labelsep=colon,strut=off} 
\usepackage{amsfonts}
\usepackage{amsmath} 

\usepackage{amsthm}
\newtheorem{lem}{Lemma}
\newtheorem{definition}{Definition}

\newcommand{\focal}[1]{\text{FOCAL}_{#1}}
\newcommand{\open}{\text{OPEN}}

\newcommand{\confspace}[2]{\mathcal{Q}^{#1}_{#2}}
\newcommand{\compconf}[1]{q_{#1}} 
\newcommand{\conf}[2]{q^{#1}_{#2}} 
\newcommand{\confstart}[1]{q^{#1}_{\text{start}}}
\newcommand{\confgoal}[1]{q^{#1}_{\text{goal}}}
\newcommand{\robot}[1]{\mathcal{R}_{#1}}
\renewcommand{\path}[2]{\pi^{#1}_{#2}} 

\title{\title{Unconstraining Multi-Robot Manipulation: Enabling Arbitrary Constraints in ECBS with Bounded Sub-Optimality}}
\author {
    Yorai Shaoul\equalcontrib,
    Rishi Veerapaneni\equalcontrib,
    Maxim Likhachev,
    Jiaoyang Li
}
\affiliations {
    Carnegie Mellon University\\
    \{yshaoul, rveerapa, maxim\}@cs.cmu.edu, 
    jiaoyangli@cmu.edu}

\usepackage{bibentry}

\begin{document}

\maketitle

\input{abstract}

\input{figure_teaser}

\input{introduction}

\input{problem_formulation}

\input{background}

\input{figure_constraints}

\input{constraints}

\input{figure_gen_ecbs_illustration}

\input{algorithmic_approach}

\input{experiments}

\input{conclusions}

\section*{Acknowledgements}
The work was supported by the National Science Foundation under Grant IIS-2328671 and the CMU Manufacturing Futures Institute, made possible by the Richard King Mellon Foundation.
\bibliography{aaai24}

\end{document}

%% file: abstract.tex
\begin{abstract}

Multi-Robot-Arm Motion Planning (M-RAMP) is a challenging problem featuring complex single-agent planning and multi-agent coordination. Recent advancements in extending the popular Conflict-Based Search (CBS) algorithm have made large strides in solving Multi-Agent Path Finding (MAPF) problems. However, fundamental challenges remain in applying CBS to M-RAMP. A core challenge is the existing reliance of the CBS framework on conservative ``complete'' constraints. These constraints ensure solution guarantees but often result in slow pruning of the search space -- causing repeated expensive single-agent planning calls. 
Therefore, even though it is possible to leverage domain knowledge and design incomplete M-RAMP-specific CBS constraints to more efficiently prune the search, using these constraints would render the algorithm itself incomplete. This forces practitioners to choose between efficiency and completeness.

In light of these challenges, we propose a novel algorithm, Generalized ECBS, aimed at removing the burden of choice between completeness and efficiency in MAPF algorithms. Our approach enables the use of arbitrary constraints in conflict-based algorithms while preserving completeness and bounding sub-optimality. This enables practitioners to capitalize on the benefits of arbitrary constraints and opens a new space for constraint design in MAPF that has not been explored. We provide a theoretical analysis of our algorithms, propose new ``incomplete'' constraints, and demonstrate their effectiveness through experiments in M-RAMP.
\end{abstract}

%% file: figure_teaser.tex
\begin{figure}[t]
  \centering\includegraphics[width=\linewidth]{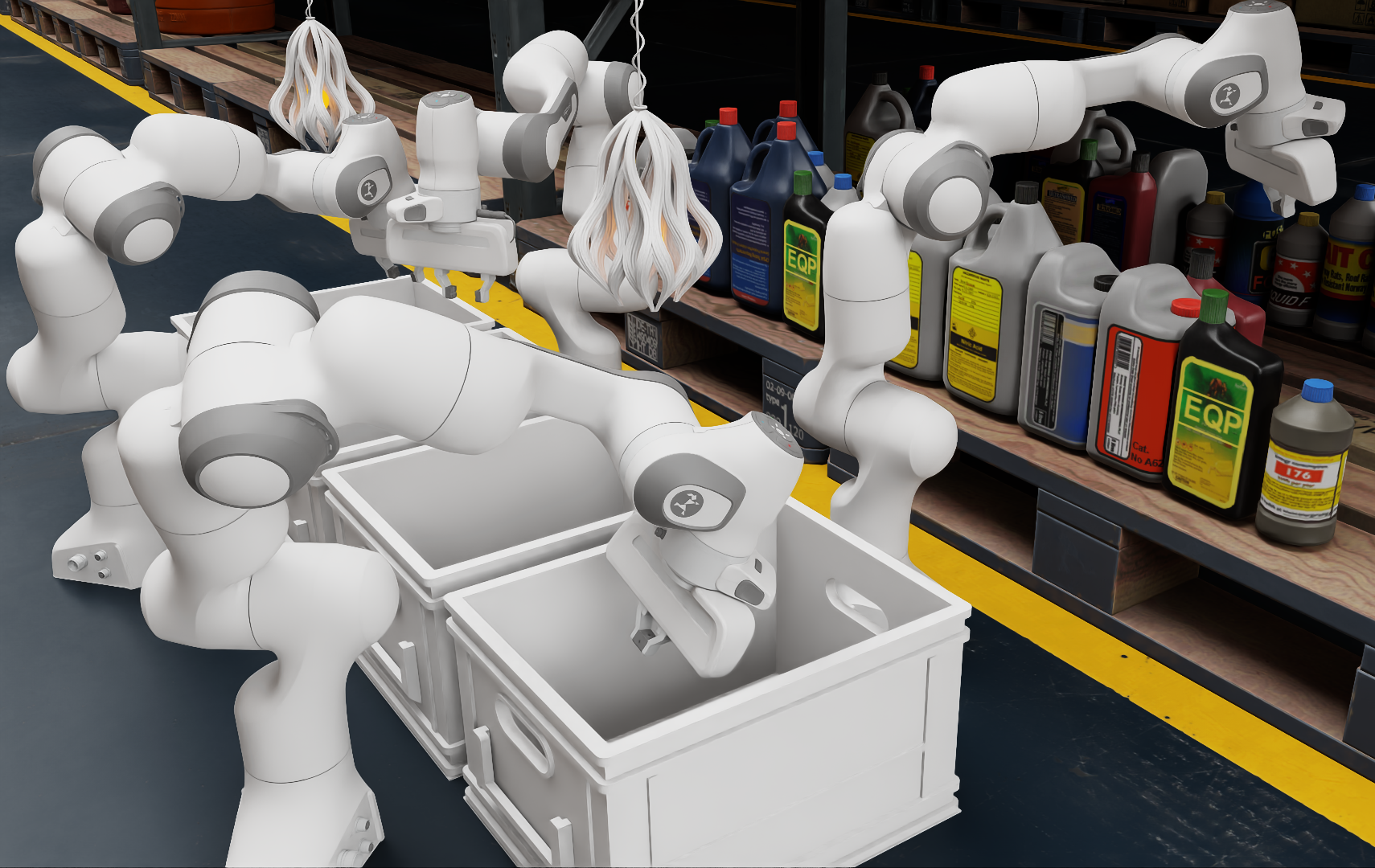}
  \caption{A team of 4 manipulators collaborating in a bin picking task. Popular MAPF algorithms such as CBS {can} be applied to multi-arm manipulation but may be ineffective due to their conservative approach to conflict resolution. Our proposed algorithm allows for more efficient planning, by capitalizing on stronger ``incomplete'' constraints,  without compromising theoretical guarantees.
  \vspace{-0.3cm}}
  \label{fig:teaser}
\end{figure}

%% file: introduction.tex
\section{Introduction}
Teams of robots operating together in a shared workspace, can enhance the efficiency of existing systems and tackle challenges beyond the capabilities of a single robot.
A key problem in multi-agent systems is Multi-Agent Path Finding (MAPF), the problem of finding efficient collision-free paths for multiple agents on graphs.
Most existing multi-agent robotics research and MAPF applications have focused on simplistic systems like planar mobile robots in warehouses. Recently, interest in using more complex multi-agent systems to complete new tasks has increased. As exemplified in Figure \ref{fig:teaser}, multi-robot-arm motion planning (M-RAMP) is particularly appealing due to its applicability for autonomous assembly, arrangement, and construction.

M-RAMP focuses on teams of high degrees of freedom (DoF) manipulators working in a shared workspace. Unlike typical MAPF applications in warehouses where hundreds of simple robots traverse a planar floor, M-RAMP typically has a few (e.g., less than $20$) robots, each having a rich configuration space (e.g., 7 DoF). Thus, M-RAMP contains unique challenges that standard 2D MAPF applications do not face. In particular, unlike 2D MAPF which assumes fast single-agent planners, M-RAMP single-agent planners need to search more complex configuration spaces without equally informed heuristics and require non-trivial collision checking.
These differences may point to MAPF algorithms being inapplicable to M-RAMP. However, as recently shown, formulating M-RAMP as an application of MAPF to robot arms is an avenue for solving it efficiently. 

The key insight in MAPF is to avoid planning in the composite configuration space for all agents by leveraging the semi-independence of agents.
MAPF methods thus contain two main components: single-agent path finding and multi-agent collision resolution. 
From an abstract level, collisions in heuristic-search-based MAPF methods are prevented or resolved using \textit{constraints}. Constraints, e.g., requiring agents $\robot{3}$ to treat $\robot{2},\robot{1}$ as moving obstacles or to avoid specific configurations at specific times, enable iterative planning of single agents to work towards a collision-free solution. 
Constraints have significant effects on the theoretical guarantees and practical results of methods.

Conflict-Based Search (CBS)~\cite{sharon2015conflict} famously applies constraints to prevent certain agents from occupying certain configurations at certain times to resolve conflicts. 
CBS's constraints are complete (Definition \ref{def:complete_constraints}) and enable CBS-based methods to yield optimal or bounded sub-optimal solutions. However, resolving collisions with these constraints may require many iterations practically causing CBS variants to often be slow to find solutions, especially when applied to domains with non-point robots. 

Priority-based methods like Prioritized Planning (PP) \cite{erdmann1987multiple} and PIBT \cite{okumura2019pibt} prevent collisions by assigning agents priority and constraining lower-priority agents to avoid higher-priority agents. Priority constraints aggressively prune the search space and often make methods employing them fast. However, this comes at the cost of occasional failures in trivial instances requiring non-trivial coordination (e.g., two agents swapping positions in a hallway). Thus, these overly strong constraints may create methods that are ``incomplete'' in that they can fail to find solutions even when those exist. 

We arrive at a central dilemma in MAPF: either gain efficiency by employing incomplete constraints or retain completeness at the cost of more computation. In challenging domains like M-RAMP, we ideally want to use \textit{arbitrary}, likely incomplete, constraints to gain efficiency, but keep completeness.
\citet{walker2020generalized} proposed to support arbitrary constraints with completeness by adding them alongside complete constraints and then probabilistically deactivating them later on. 
Orthogonal to this approach, we discriminate between different types of arbitrary constraints to allow effective combinations of these, and guarantee bounded sub-optimal solution costs.


To this end, we design a framework that incorporates arbitrary constraints in CBS and identifies effective constraint types while guaranteeing completeness and bounded sub-optimality. We instantiate this framework with the algorithm Generalized ECBS, and propose new constraints\footnote{We highlight that our work only scratches the surface of constraint design and enables an entirely new design space for future research that can leverage different domain knowledge, or collected data, to construct new constraints.} effective in M-RAMP. Across experiments with 4, 6, and 8 robot arms, we demonstrate how M-RAMP can be better solved with Generalized ECBS compared to complete methods (e.g., ECBS) and incomplete methods (e.g., Prioritized Planning, or ECBS with only incomplete constraints).

%% file: problem_formulation.tex
\section{M-RAMP Problem Definition}

Multi-Robot-Arm Motion Planning (M-RAMP) is the problem of finding low cost and collision-free paths for teams of manipulators (also named \textit{agents}) $\robot{1},\dots, \robot{n}$ from start configurations $\conf{i}{\text{start}}$ to goal configurations $\conf{i}{\text{goal}}$ in a shared workspace. 
Let us consider $\confspace{i}{} \subseteq \mathbb{R}^d$ as the configuration space of a single robot $\mathcal{R}_i$ with $d$ DoF and $\confspace{i}{\text{free}} \subseteq \confspace{i}{}$ the set of configurations not colliding with obstacles.
A configuration $\conf{i}{} \in \confspace{i}{}$ is defined by assigning values (joint angles) to all the DoF. We denote the occupancy of a robot $\robot{i}$ taking on the configuration $\conf{i}{}$ as $\robot{i}(\conf{i}{}) \subset \mathbb{R}^3$. Agents $\robot{i}$ and $\robot{j}$ are defined to be in \textit{conflict} when $\robot{i}{}(\conf{i}{}) \cap \robot{j}{}(\conf{j}{}) \neq \emptyset$.

A solution to the (discrete-time) M-RAMP problem is a set of single-agent paths $\Pi = \{\pi^1, \pi^2, \dots, \pi^n\}$. Each path\footnote{We use the shorthand $\conf{i}{t} = \path{i}{}[t]$ to denote the configuration that $\robot{i}$ takes at time $t$ according to $\path{i}{}$ when the context is clear. For other subscripts time is specified as needed.} $\pi^i = \{\conf{i}{0}, \conf{i}{1}, \dots, \conf{i}{T_i}\}$ obeys $\conf{i}{0}=\conf{i}{\text{start}}$, $\conf{i}{T_i}=\conf{i}{\text{goal}}$, all $\conf{i}{t} \in \confspace{i}{\text{free}}$, and there are no conflicts between any pair $\robot{i}, \robot{j}$ at all time steps or interpolated transitions when those follow $\pi^i$ and $\pi^j$.
Our objective is to find the best valid solution in terms of time and motion (radians), minimizing the sum of costs $|\Pi| = \sum_{i=1}^n |\pi^i| = \sum_{i=1}^n \sum_{t=1}^{T_i} \text{cost}(\conf{i}{t-1}, \conf{i}{t})$.

\label{sec:prob}

%% file: background.tex
\section{Background}
\label{sec:background}

In practice, a common approach to solve M-RAMP is to plan in the composite state-space (i.e., considering all agents as a single high-dimensional agent) using sampling-based planners such as Rapidly exploring Random Trees (RRT) and its extensions \cite{rrt, RRT-Connect, shome2020drrt}. 
With more than one agent, solution quality sharply decreases, especially in clutter.

Another approach is to apply MAPF solvers for M-RAMP. Since naively applying algorithms like CBS (complete) or PP (incomplete) to arms may result in poor performance (Sec. \ref{sec:experiments}), more nuanced approaches have been proposed. For example, CBS-MP \cite{solis2021representation} utilizes sparse roadmaps and modifies CBS constraints to prune the search space more but causes them to become incomplete.\footnote{Although the original paper claims completeness, we and its authors concluded otherwise; see \citet{shaoul2024xcbs} for details.} 
In this study, we aim to broaden CBS-based methods by supporting incomplete constraints while maintaining completeness.
We first describe how manipulation planning can be solved using graph search and then describe the relevant CBS and ECBS work that our method builds on.

\subsection{Manipulation Planning as Graph Search}
One way to realize motion planning for a single agent $\robot{i}$ is by encoding its configuration space $\confspace{i}{}$ in a graph and searching it for a collision-free shortest path from $\conf{i}{\text{start}}$ to $\conf{i}{\text{goal}}$ via a sequence of edge-connected vertices. Vertices in this graph denote configurations $\conf{i}{}$, and edges are valid transitions between configurations. Since it is generally infeasible to enumerate all collision-free vertices and edges on these graphs for any reasonable $d$, it is hard to compute informative heuristics.
In this study, we employ methods to search these graphs as outlined by \citet{cohen2011adaptiveprimitives}. Specifically, we use Weighted-A* for graph search to discourage excessive exploration of the vast search space, construct edges as discretized arm motions, and determine edge and vertex validity during the search.

Coordination between multiple agents can be done by augmenting single-agent graphs with a discretized time dimension, allowing waiting, and ensuring that no conflicts occur. We call conflicts during edge traversals \textit{edge conflicts}, and \textit{vertex conflicts} otherwise.
By imposing this graph structure on the motions of individual robot arms, we can directly view M-RAMP as an application of MAPF to the robot-arm domain and employ algorithms like CBS to solve it.

\subsection{Conflict-Based Search (CBS)}
CBS is a popular complete and optimal MAPF solver that employs a low-level single-agent planner and a high-level constraint tree (CT) to resolve conflicts \cite{sharon2015conflict}. Given a set of paths, CBS detects conflicts in pairs of agents and resolves them by applying constraints.

Concretely, a CT node $N$ contains a set of paths $N.\Pi$, one for each agent, which satisfy a set of constraints $N.C$. CBS prioritizes CT nodes in an OPEN queue by their cost $N.\text{cost} = |N.\Pi|$. CBS initializes OPEN with a root CT node with no constraints and individually optimal paths. CBS proceeds iteratively, taking the best node and checking for conflicts. If there are no conflicts, it has found a valid solution and terminates. If not, it picks a conflict to resolve.
CBS resolves conflicts by applying constraints to the participating agents, one for each, and generates two successor CT nodes with agents replanned to satisfy the new constraint.

\subsection{Enhanced Conflict-Based Search (ECBS)}

ECBS \cite{barer2014suboptimal} differs from CBS in its use of a focal list \cite{pearl1982astarepsilon}, a mechanism for arbitrarily prioritizing a portion of OPEN while guaranteeing bounded sub-optimality, in the high- and low-level search.
In the high-level search, ECBS defines 
$$\focal{} := \{N \in \open \;\big\vert\; N.\text{cost} \leq w \cdot \min_{N' \in \open} lb(N')\}$$
where $lb(N):= \sum_{i=1}^{n} lb(\pi^{i})$ and $lb(\pi^i)$ is a lowerbound on the cost of path $N.\Pi[i]$. ECBS chooses the CT node $N$ with the least conflicts, $N = \arg \min_{N' \in \focal{c}}|N'.\Omega|$, with $N.\Omega$ being the conflict set identified in $N.\Pi$. This yields solutions with cost at most $w \cdot C^*$, with $C^*$ the optimal sum of costs, that are generally found much faster than CBS. 

Crucially for guaranteeing completeness in CBS and its variants, the constraints that they employ must be \textit{complete} (otherwise known as \textit{mutually disjunctive}). We adopt the following definition from \citet{li2019multi}.
\begin{definition}[Complete Constraints] A constraint is complete if, when used to resolve a conflict between $\robot{i}$ and $\robot{j}$ via imposing constraints $c_i$ and $c_j$ on the agents respectively, there do not exist any two conflict-free paths for $\robot{i}$ and $\robot{j}$ with $\robot{i}$ violating $c_i$ and $\robot{j}$ violating $c_j$.
\label{def:complete_constraints}
\end{definition}
\noindent For example, vertex constraints (Sec. \ref{sec:constraints}) are complete. Let vertex constraints $c_i$ and $c_j$ be imposed on $\robot{i}$ and $\robot{j}$ to resolve a conflict between $\robot{i}(\conf{i}{t})$ and $\robot{j}(\conf{j}{t})$ at time $t$. Any path for $\robot{i}$ that violates $c_i$ must include $\conf{i}{t}$ at time $t$. $\robot{j}$ mirrors. Therefore, these paths will always lead to a collision at time $t$, and by Def. \ref{def:complete_constraints}, vertex constraints are complete.

%% file: figure_constraints.tex
\begin{figure*}[t]
  \begin{minipage}{\textwidth}
      \centering
    \includegraphics[width=\textwidth]{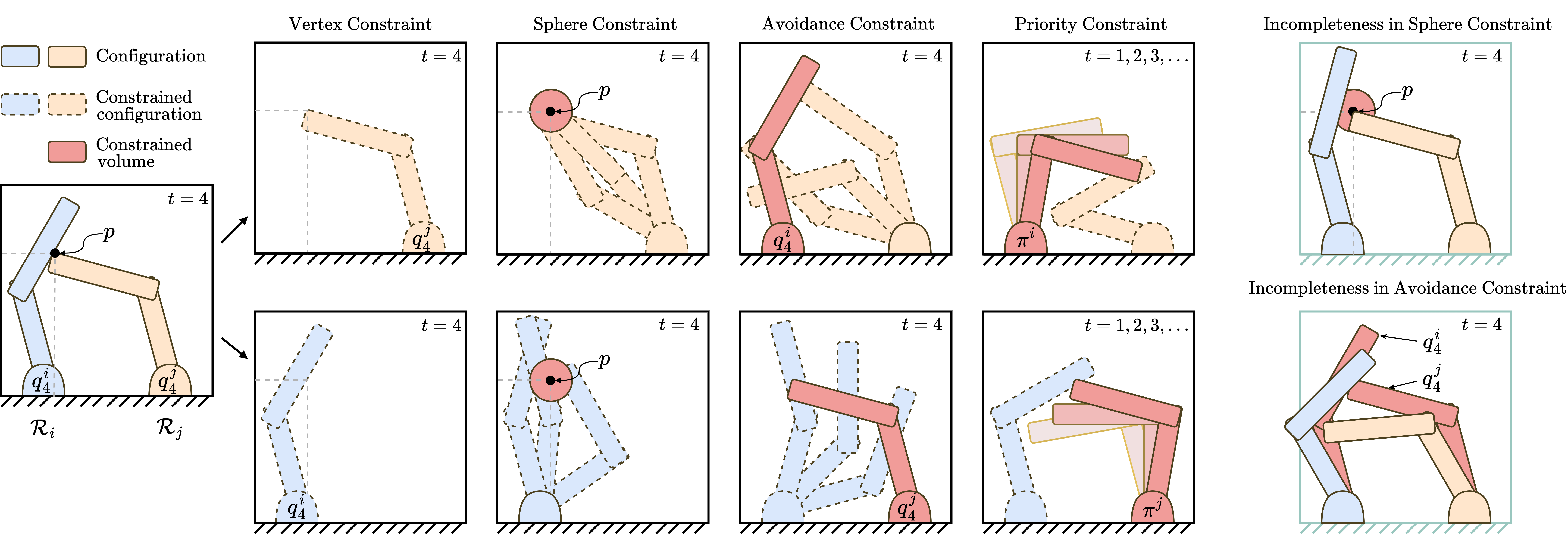}
    \caption{Given agent $\robot{i}$ in $\conf{i}{4}$ and $\robot{j}$ in $\conf{j}{4}$ conflicting at time $t=4$ and point $p$ (leftmost), we illustrate the constraint landscape when replanning for $\robot{j}$ (top row) and for $\robot{i}$ (bottom row) alongside examples of invalid configurations under the constraint (marked with dashed outlines). When applicable, we include agent configurations (e.g., $\conf{i}{4}$) or sequence of configurations (e.g., $\pi^i$) in the robot base link. From left to right: {vertex constraints} forbid an agent from taking on its conflicting configuration at $t$. {Sphere constraints} forbid collisions with a sphere centered at $p$ at time $t$. {Avoidance constraints} disallow collisions with the conflicting configuration of the other agent at $t$. {Priority constraints} force an agent to plan around the current path of the other. Rightmost: examples of incompleteness in the sphere and avoidance constraints. We illustrate valid \textit{conflict-free} configurations between $\robot{i}$ and $\robot{j}$ where each invalidates its imposed constraints. This scenario shows that sphere and avoidance constraints are not mutually disjunctive, and therefore are not complete within CBS.}
    \label{fig:constraints}
  \end{minipage}%
\end{figure*}

%% file: constraints.tex
\section{Constraints}
\label{sec:constraints}

It is often intuitive to formulate effective constraints that accelerate the search but are incomplete. Historically, methods either accepted this incompleteness in favor of the gains in efficiency it practically provided or agreed to bargain runtime for the theoretical guarantee of completeness.
In this work, we allow for the use of arbitrary constraints in the CBS framework while guaranteeing completeness and bounding sub-optimality. 

This section describes the constraints commonly used in MAPF and introduces new constraints. In the following definitions, we consider the constraints $c_i$ and $c_j$ that are imposed to resolve a conflict $\omega$ between agent $\robot{i}$ and $\robot{j}$. The time of conflict is $t$, at which the robots took on the configurations $\conf{i}{\omega}$ and $\conf{j}{\omega}$ with the associated collision point $p \in \robot{i}(\conf{i}{\omega}) \cap \robot{j}(\conf{j}{\omega})$. For brevity, we describe the resolution of vertex conflicts and note that edge conflicts in $(t, t+1)$ follow identically. 
Additionally, we only describe $c_i$ with $c_j$ mirroring it accordingly, similar to how CBS resolves conflicts with symmetric pairs of constraints, one per agent. 

\subsection{Existing Constraints}
The most commonly used constraints in the CBS framework are \textit{vertex} and \textit{edge} constraints. These are the original domain-agnostic constraints in CBS that guarantee completeness. 

\begin{definition}[Vertex Constraint] $c_i$ forbids $\robot{i}$ from taking on the configuration $\conf{i}{\omega}$ at time $t$.
\label{def:vertex_constraint}
\end{definition}
\noindent The main drawback of these constraints is their small effect on the search space: preventing agents from occupying exactly one configuration (vertex) at a time. We would prefer to remove larger subsets of the search space when resolving a conflict, but this may come at the cost of completeness.
For example, instead of solely avoiding a single configuration, CBS-MP avoids the entire volume occupied by the other agent. We name this an ``avoidance" constraint.
\begin{definition}[Avoidance Constraint]
$c_i$ forbids $\robot{i}$ from colliding with the volume $\robot{j}(\conf{j}{\omega})$ at time $t$.
\end{definition}
\noindent Interestingly as we show in Figure \ref{fig:constraints}, this is not complete.

Arguably the most popular incomplete constraint, used by PP and PBS \cite{ma2019pbs}, is ``priority.''
\begin{definition}[Priority Constraint]
$c_i$ forbids $\robot{i}$ from colliding with $\robot{j}$ as it moves along its current path $\pi^j$.
Formally, $\forall t \in \{0,\cdots T_j\}$ $c_i$ forbids $\robot{i}$ from taking on $\conf{i}{}$ at time $t$ if $\robot{i}(\conf{i}{}) \cap \robot{j}(\conf{j}{t}) \neq \emptyset$. 
\end{definition}

\subsection{New Constraints}
In M-RAMP, a collision point $p$ is associated with a conflict. A natural choice for a constraint then is restricting agents from re-occupying $p$ at $t$. Adding a margin, we call these \textit{sphere} constraints.
\begin{definition}[Sphere Constraints] 
$c_i$ forbids $\robot{i}$ from colliding with a spherical obstacle $S^2(p,r)$ with a radius $r$ centered at the collision point $p$ at time $t$. i.e., the set $\{\conf{i}{} \in \confspace{i}{\text{free}} \;\vert\; \robot{i}(\conf{i}{}) \cap S^2(p,r) \neq \emptyset\}$ is disallowed at time $t$. 

\end{definition}

\noindent Depending on the collision checker of choice, sphere constraint satisfaction can be cheap to compute and an effective way to resolve conflicts quickly. 
With a non-zero radius, sphere constraints are incomplete (see Fig. \ref{fig:constraints}), however, as $r \rightarrow 0$ (i.e., as the constraint becomes the point $p$ itself) a sphere constraint becomes complete. We provide a quick proof.
Let $\conf{i}{}$ and $\conf{j}{}$ be any two configurations violating a symmetric ``point constraint.'' Formally, $p \in \robot{i}(\conf{i}{})$ and $p \in \robot{j}(\conf{j}{})$. Thus, $\robot{i}(\conf{i}{}) \cap \robot{j}(\conf{j}{}) \neq \emptyset$ and ``point constraints'' are complete by Definition \ref{def:complete_constraints}.

Additionally, we add flexibility to priority constraints by restricting priorities to a single time $t$.
\begin{definition}[Step-Priority Constraints]
Let $\conf{j}{t}$ be the configuration of $\robot{j}$ at time $t$ in its current path $\pi^j$.
$c_i$ forbids $\robot{i}$ from colliding with $\robot{j}(\conf{j}{t})$ at time $t$. 
\end{definition}
\noindent Step-priority constraints are different from avoidance constraints. When $\robot{i}$ \textit{avoids} $\robot{j}$, it is not allowed to collide with the original conflicting configuration of $\robot{j}$. Step-priority disallows $\robot{i}$ from colliding with the updated configuration of $\robot{j}$ according to its most recent path; therefore, this constraint adaptively changes with $\path{j}{}$.

%% file: figure_gen_ecbs_illustration.tex
\begin{figure*}[t]
  \begin{minipage}{\textwidth}
    \centering
      \includegraphics[width=\linewidth]{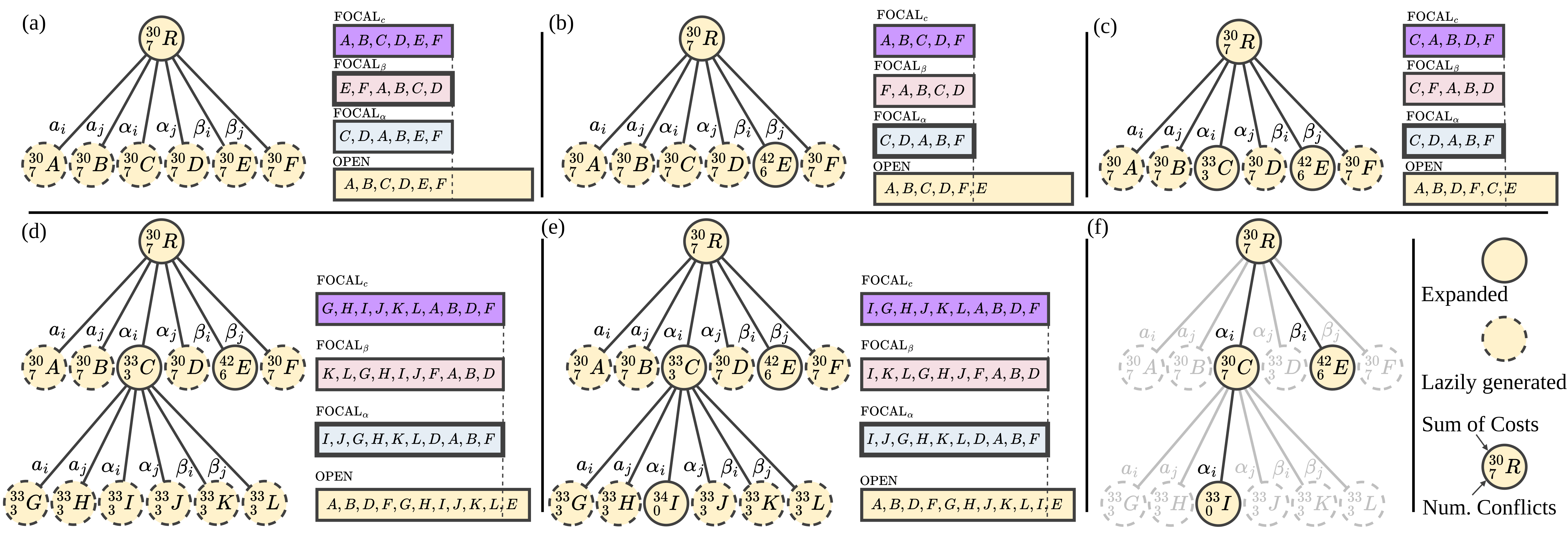}
    \caption{Illustration of the CT and priority queues in Generalized ECBS. Each CT node shows its number of conflicts (left subscript), and sum of costs (left superscript). The active priority queue sampled by DTS has a {bold} perimeter. (a) After the root node $R$ has been expanded, its children are generated lazily (dashed). (b) Node $E$ is chosen from $\text{FOCAL}_\beta$ and evaluated. Upon DTS resampling, $FOCAL_\alpha$ is activated. (c) Node $C$ is chosen and evaluated. (d) Node $C$ was chosen (a second time) and lazily expanded. (e) Node $I$ is chosen and evaluated. (f) Node $I$ is expanded and marked as a goal. We note that the one-step-lazy evaluations may allow for significantly reduced work relative to naively evaluating all child nodes upon expansion -- an operation that does not scale with the number of constraints.\vspace{-0.3cm}}
    \label{fig:gencbs_illustration}
  \end{minipage}
\end{figure*}

%% file: algorithmic_approach.tex

\section{Algorithmic Approach}
Generalized ECBS allows CBS variants to handle arbitrary constraints while keeping completeness and bounded sub-optimality guarantees. 
We first describe a simple and inefficient method named Arbitrary Constraint ECBS (AC-ECBS) for incorporating arbitrary constraints into ECBS. We prove that AC-ECBS is complete and bounded sub-optimal and lay the theoretical foundation for Generalized ECBS. We then introduce Generalized ECBS and explain how it overcomes the efficiency challenges plaguing AC-ECBS by lazily expanding CT nodes and prioritizing promising constraint types.

\subsection{Arbitrary Constraint ECBS with Guarantees}
\label{sec:ac_ecbs}

Given $K$ arbitrary \textit{constraint types}, i.e., $K$ different constraint options to resolve the same conflict (for example, $k=3$ when including small sphere, large sphere, and avoidance constraints), 
AC-ECBS is identical to ECBS except that it generates $B=2K + 2$ child CT nodes on each expansion. $2K$ nodes are generated using the arbitrary constraints and the additional two with complete (vertex/edge) constraints.

\begin{lem}
\label{lem:completeness_ac_ecbs}
    AC-ECBS is complete.
\end{lem}
\begin{proof}
In ECBS, a solution is found in an iterative process. Starting from a root CT node $N_R$, CT nodes are branched to two child nodes by adding a single complete (vertex or edge) constraint to each and queuing them in a high-level OPEN list. In AC-ECBS, we know that for each expanded CT node, there always exist two child nodes that have been created by the addition of only complete constraints. Define a ``Complete Subtree'' $\mathcal{T}_C := \{N \in \text{OPEN} \mid  N.C \text{ only has complete constraints}\}$. Therefore, the high-level OPEN list of AC-ECBS is a superset of the OPEN of ECBS and therefore includes CT nodes guaranteed to lead to a solution by continuing to resolve conflicts in them with complete constraints. Thus, additional child nodes in CT expansions does not affect the completeness of ECBS.
\end{proof}

\begin{lem}
\label{lem:bound_ac_ecbs}
    AC-ECBS is bounded sub-optimal with a bound $w$ identical to ECBS.
\end{lem}
\begin{proof}
As before, let the lower-bound of a CT node $N$ be $lb(N)$. 
From Lemma \ref{lem:completeness_ac_ecbs}, the high-level OPEN in AC-ECBS is always a superset of some ECBS OPEN, which we'll name $\overline {\text{OPEN}} \subseteq \open{}$. Therefore, the lower bound $\text{OPEN}_{lb} = \min_{N \in \text{OPEN}} lb(N) \leq \min_{N \in \overline{\text{OPEN}}} lb(N) =\overline{\text{OPEN}}_{lb}$. 
In ECBS, we have $\overline{\text{OPEN}}_{lb} \leq C^*$ by construction. 
Thus, $\text{OPEN}_{lb} \leq C^*$ and all nodes $N$ in the AC-ECBS FOCAL list have bounded sub-optimal costs $N.\text{cost} \leq w \cdot C^*$. \qedhere

\end{proof}

\begin{algorithm}[!t]
    \caption{Generalized ECBS High-level Planner}
    \label{alg:high-level}
    \footnotesize
    \SetKwInOut{Input}{Input}
    \SetKwInOut{Output}{Output}
    \SetKwFunction{LLPlanner}{\scriptsize LLPlanner}
    \SetKwFunction{Solve}{\scriptsize Solve}
    \SetKwFunction{GetCost}{\scriptsize GetCost}
    \SetKwFunction{InitRootNode}{\scriptsize InitRootNode}
    \SetKwFunction{InitDTS}{\scriptsize InitDTS}
    \SetKwFunction{BiasDTS}{\scriptsize BiasDTS}
    \SetKwFunction{DTS}{\scriptsize DTS}
    \SetKwFunction{Reward}{\scriptsize Reward}
    \SetKwFunction{Penalize}{\scriptsize Penalize}
    \SetKwFunction{RemoveTime}{\scriptsize RemoveTime}
    \SetKwFunction{FindConflicts}{\scriptsize FindConflicts}
    \SetKwFunction{GetSuccessors}{\scriptsize GetSuccessors}
    \SetKwFunction{GetConstraints}{\scriptsize GetConstraints}
    \SetKwFunction{ConflictsToConstraints}{\scriptsize ConflictsToConstraints}
    \SetKwFunction{CopyCTNode}{\scriptsize CopyCTNode}
    
    \Input{$n$: Number of agents \newline
        $\confstart{} = \{\confstart{0}, \dots \confstart{n} \}$ \newline
        $\confgoal{} = \{\confgoal{0}, \dots \confgoal{n} \}$ \newline}
    \Output{Path $\Pi = \{\pi^1, \cdots, \pi^n\}$ from start to goal states.}
    \SetAlgoLined\DontPrintSemicolon

    \SetKwFunction{procroot}{InitRootNode}
    \SetKwProg{myprocroot}{Procedure}{}{}
    
    \SetKwFunction{proc}{Plan}
    \SetKwProg{myproc}{Procedure}{}{}

    \vspace{5pt}
    \myprocroot{\procroot{}}{
    RootNode.$C$ $\leftarrow$ $\emptyset$ \tcp*{No initial constraints.}
    RootNode.$\Pi$ $\leftarrow$ invoke \LLPlanner for each agent \; 
    RootNode.cost $\leftarrow$ \GetCost(RootNode.$\Pi$) \;
    \KwRet RootNode\;  
    } 
    \vspace{5pt}

    \myproc{\proc{$n$, $\confstart{}$, $\confgoal{}$}}{ 
    RootNode $\leftarrow$ \InitRootNode( ) \;
    \InitDTS( ) \; 
    \BiasDTS() \tcp*{Optionally prioritize certain focal lists.}
    OPEN.insert(RootNode) \;
    \While{OPEN not empty}{
    
    $k \gets$ \DTS() \;
    FOCAL$_k$ $\gets \{N | N\text{.cost} \leq w \cdot \min\limits_{N' \in \text{OPEN} } N_{LB} \}$ \; \label{line:focal-alpha}
    $N \gets \;^{\; \arg \min}_{N \in \text{FOCAL}_k} f_k(N)$\; \label{line:focal-priority-alpha} \vspace{0.1cm}
    
    \If{$N$.\textit{agentsReplan} $\neq \emptyset$}{ \label{line:lazy-replan}
    \tcp*{Lazily generated nodes are evaluated.}
        \For{$\robot{j} \in N$.\textit{agentsReplan}}{
            $N.\Pi[j] \gets$ Invoke \LLPlanner for $\robot{j}$ \;
        }
        $N$.cost  $\gets$ \GetCost($N.\Pi$) \; \label{line:lazy-cost-update}
        $N$.\textit{agentsReplan} $\gets \emptyset$ \;
        
        $\hat {\Omega} \gets N.\Omega$\tcp*{{Previously found conflicts.}}
        $N$.$\Omega \leftarrow$ \FindConflicts($N$.$\Pi$) \label{line:conflicts_update}\;  
        \If{$|N$.$\Omega| < |\hat {\Omega}|$}{
            \DTS.\Reward($k$) \label{line:dts_reward} \tcp*{{Increase prob. for sampling $k$.}}
        }
        \Else{
            \DTS.\Penalize($k$)\label{line:dts_penalize} \tcp*{{Reduce prob. for sampling $k$.}}
        }
        \textbf{continue}
    }

    OPEN.pop($N$)\;
    \If{$N$.$\Omega = \emptyset$}{
        \KwRet $N$.$\Pi$\;
    }
    $C \leftarrow$ \GetConstraints ($\Omega$.first) \tcp*{Constraints for the conflict, complete and incomplete.}
    
    \For{$c \in C$}{
        $N' \gets$ \CopyCTNode($N$)\;
        $N'$.$C$ $\leftarrow$ $N.C$ $\cup$ \{c\}\; \label{line:add-constraint}
        $N'$.\textit{agentsReplan} $\gets c$.agent\_id \; \label{line:delay-replanning}
        OPEN.insert($N'$) \tcp*{One step lazy generation.}
    }
    }
    \KwRet $\emptyset$    
    }
\end{algorithm}

\subsection{Generalized ECBS}
\label{sec::gen_ecbs}
Although AC-ECBS can handle arbitrary constraints, it suffers from needing to generate all $B$ successors at every CT expansion.
Furthermore, certain constraints may be more appropriate for certain planning problems (e.g., smaller constraints are better than larger ones in tight coordination), and we would like the algorithm to adapt to such a need.
We tackle both of these problems with Generalized ECBS. Our key idea is to use lazy expansions to ease computational work and multiple focal queues to adaptively prioritize constraint types. Algorithm \ref{alg:high-level} describes Generalized ECBS.

\subsubsection{Lazy CT Expansion}

Imagine we are given a node with some conflicts that need to be resolved. In AC-ECBS, we apply the $B$ constraints and query 
that many low-level planning calls 
before inserting these $B$ successor nodes into OPEN and FOCAL. 
Drawing inspiration from work on lazy planning, where nodes are generated during the search and checks for their validity are deferred \cite{mukherjee2022epase, haghtalab2018lazysp}, we instead generate all $B$ nodes, each with an extended constraint set (Line \ref{line:add-constraint}), but delay the low-level planning calls (Line \ref{line:delay-replanning}). Unlike common lazy approaches that assume fast node generation, our successor generation requires replanning and is expensive. Thus, we reuse parent values (e.g., the sum of costs and number of conflicts) in child nodes and insert them into the queue using the priority derived from these approximate values. Now, when choosing a node for expansion, if it is an approximate node we only then call the low-level planner to get the updated paths and re-queue it (Line \ref{line:lazy-replan}).

When adding lazy evaluations to AC-ECBS directly, successor nodes are indistinguishable in terms of their cost and conflict count values (see Fig. \ref{fig:gencbs_illustration}(a)), causing child nodes to be prioritized in an uninformed manner. We can solve this issue using multiple focal queues.

\subsubsection{Multiple Focal Queues}
The main problem with one focal queue is that it can only be sorted with one priority function, making it hard to distinguish between lazy successor nodes with the same approximate values (i.e., a copy of their parent values). Additionally, designers creating these constraints may have priors on certain constraints being more useful than others. Inspired by \citet{dtsMultiHeuristic}, we can tackle both of these issues by generalizing CBS to have multiple FOCAL queues (Line \ref{line:focal-alpha}), each sorted by their priority function. Each priority function can distinguish between identical lazy successor nodes and incorporate external domain knowledge by prioritizing specific constraint types.

Concretely, given $K$ arbitrary constraint types (excluding the complete edge/vertex constraints), we initialize $K+1$ FOCAL queues. Each $FOCAL_k$ uses $f_k(N)$ to sort CT nodes $N$ lexicographically by $(|N.\Omega|, |N.\Pi|, \tilde \rho_k(N))$ (Line \ref{line:focal-priority-alpha}). That is, primarily according to the number of conflicts in the paths of $N$, breaking ties according to the sum of costs and (one minus) the density of constraints of type $\alpha$ in $N$.
\begin{equation*}
    \tilde \rho_k(N) := 1 - \dfrac{|\{c \in N.C \; | \; c.\text{ type } = k \}|}{|N.C|} 
\end{equation*}
We note that arbitrary priority functions are possible and encourage incorporating domain knowledge here too. We found that this worked well in practice.

\subsubsection{Choosing Between Queues}
Given multiple FOCAL queues and an OPEN queue, we need to choose which queue to pick to decide our next CT node to expand. A simple approach used before in single-agent works is to evenly rotate between all the queues \cite{aine2016mhastar}. 

However, some constraints may prove to be more effective in different parts of the search, and we would like the algorithm to dynamically prioritize those. We can use Dynamic Thompson Sampling (DTS) as described in \citet{dtsMultiHeuristic} to define a selection mechanism among all the queues that get updated during the search. As seen in Algorithm \ref{alg:high-level},
each time an \textit{un-evaluated} CT node $N$ is considered from $\text{FOCAL}_k$, it is evaluated and its true conflict set $\Omega$ is computed (line \ref{line:conflicts_update}). If it has fewer conflicts than in $\hat \Omega$ inherited from its parent, then this $\text{FOCAL}_k$ is prioritized. Otherwise, it is penalized (lines \ref{line:dts_reward} and \ref{line:dts_penalize}). Internally, DTS keeps a beta distribution $\text{Beta}(\alpha_k, \beta_k)$ parameterized by $\alpha_k, \beta_k \in \mathbb{R}^{+}$. Upon each penalty or reward, DTS increments $\alpha$ or $\beta$, skewing the distribution towards $1$ or $0$ respectively. After each update operation, a value $v_k$ is sampled from all $\text{Beta}(\alpha_k, \beta_k)$, and the index $k$ of the newly activated $\text{FOCAL}_k$ is set to $\arg \max_{k} v_k$. DTS dampens the effects of the history on updates via a parameter $C$, keeping the sum $\alpha_k + \beta_k \leq C$ always. In our experiments $C=10$.

\subsubsection{Putting it all together}
We now describe Generalized ECBS with both these components (lazy expansion and multiple FOCAL queues) in a quick example (Fig. \ref{fig:gencbs_illustration}). Imagine working in a 3D manipulation domain and designing $2$ types of constraints $\alpha$ and $\beta$ to resolve conflicts. We again note that we always have access to at least one type of ``complete'' constraint (e.g., regular vertex/edge constraint). Generalized ECBS creates 4 queues; OPEN along with $\text{FOCAL}_\alpha$, $\text{FOCAL}_\beta$, $\text{FOCAL}_c$. 
$\focal{c}$ is sorted by conflict count and each $\text{FOCAL}_{\{\alpha,\beta\}}$ is sorted as described previously. We could initialize our DTS beta distributions with a uniform prior ($\alpha = \beta = 1$) or instead bias them based on experience. In this example, DTS initially favors $\text{FOCAL}_\alpha$.

In the beginning, we generate the root node and collect all collisions. Instead of generating the $6$ successor nodes by calling low-level planners for each constraint, we generate the $6$ nodes without replanning agents and insert them into OPEN and all FOCAL queues (Fig. \ref{fig:gencbs_illustration}(a)). Due to the bias, we sample $\text{FOCAL}_\alpha$ and choose its best node (Fig. \ref{fig:gencbs_illustration}(b)). Here the node contains approximate values, so we call the low-level planners to replan with the $\alpha$ constraint and re-queue the node (Fig. \ref{fig:gencbs_illustration}(c)). We reward the DTS sampling distribution since the number of conflicts decreased. We now repeat and sample $\text{FOCAL}_\alpha$. If the same node is picked (Fig. \ref{fig:gencbs_illustration}(c)), it is fully evaluated which means we generate the lazy successors (Fig. \ref{fig:gencbs_illustration}(d)). We repeat this process iteratively until we reach a conflict-free solution or timeout.
We note that Generalized ECBS is complete and bounded sub-optimal for the same reasons outlined in Section \ref{sec:ac_ecbs}.

\label{alg_app}

%% file: experiments.tex
\section{Experiments}

To evaluate the performance of Generalized ECBS, we created three test scenarios similar to real-world multi-arm manipulation problems that M-RAMP algorithms would be expected to solve (Fig. \ref{fig:experimental_evaluation}). Our tests feature closely interacting manipulators and complex obstacle landscapes. The results suggest that effectively capitalizing on the benefits of incomplete constraints helps solve M-RAMP problems faster.

\input{figure_constraints_replacement}

\subsection{Experimental Setup}
We set up $150$ planning problems, each characterized by start-goal pairs. Single-agent configurations $\conf{i}{\text{start}}$ and $\conf{i}{\text{goal}}$  were randomly chosen from a set of task-relevant poses (e.g., pick poses inside different bins) for each agent and verifying the validity of $\compconf{\text{start}} = \{\conf{1}{\text{start}}, \dots \conf{n}{\text{start}}\} $ and $ \compconf{\text{goal}} = \{\conf{1}{\text{goal}}, \dots \conf{n}{\text{goal}}\}$. This task dependence in the problem construction promotes high levels of interaction between agents, similar to that in realistic multi-arm manipulation tasks.

To shed light on the scalability of our algorithm, we vary the number of robots and the levels of interaction between experiments, as illustrated in Figure \ref{fig:experimental_evaluation}. We include two scenarios resembling real-world planning tasks: pick-and-place with $4$ robots and shelf rearrangement with $8$ robots. Both showcase significant clutter and proximity between agents. We also include a relatively open environment with $6$ robots with a single obstacle between them.

Aiming to verify the ability of Generalized ECBS to handle a large number of incomplete constraints, for each vertex or edge conflict in our evaluation, we generate \textit{avoidance}, \textit{step-priority}, \textit{complete} (vertex/edge), and three types of \textit{sphere constraints} with radii $5$, $15$, and $30$ cm. We set an initial DTS bias towards the smaller sphere constraints.

Each robot in our experiments is a Franka Panda manipulator with 7 DoF. The experiments were conducted on an Intel Core i9-12900H laptop with 32GB RAM (5.2GHz). We implemented all algorithms in \texttt{C++}, used the MoveIt!2 \cite{coleman2014reducing} software for interacting with manipulators, and Flexible Collision Library (FCL) \cite{pan2012fcl} for collision checking. 

\subsection{Evaluated Methods}
We compare Generalized ECBS to ubiquitous methods commonly used in motion planning for robotics manipulators, as well as to popular MAPF algorithms. We choose to include RRT-Connect and PRM as these are arguably the most frequently used algorithms in planning for manipulation. These algorithms plan for all agents jointly, treating the team of robotic manipulators as a single composite agent. We use the OMPL \cite{OMPL2012} implementation.

We additionally include results from MAPF planners. We include PP, CBS-MP, CBS, and ECBS, as well as Generalized CBS and Generalized ECBS. Generalized CBS differs from Generalized ECBS in that it does not assume access to conflict counts in high-level and low-level nodes, and as such does not prioritize nodes based on this information. 
All algorithms employ Weighted-A* as their single-agent planner with a weight\footnote{Our heuristic underestimates the cost to go in radians, and edges are unit cost. The weight scales the heuristic value to match the cost of edge transitions and inflates it.} of $50$ and an $L_2$ joint-angle heuristic function. Adaptive motion primitives are of $\pm10$ degrees in individual joints when the end-effector is closer than $20$ cm to its goal location. Otherwise, $\pm15$ degrees in the lower $4$ joints. Edge transitions are of uniform cost and take a single timestep. The sub-optimality bound of ECBS and Generalized ECBS is set to 1.3, and for all FOCAL lists for Generalized CBS to $1$.
Our implementation of CBS-MP differs
slightly from the original in that, here, agents plan on implicit graphs and not on precomputed roadmaps to compare all search algorithms on the same planning representation.

\input{figure_experiments_all_constraints_naive}


We evaluate algorithm scalability and solution quality per scene by reporting mean and standard deviation for planning time and solution cost across problems, along with the success rate for each algorithm. Algorithms are allotted $60$ seconds for planning; exceeding this time is considered a failure. Solution cost is measured by the total joint motion in radians. 
All solutions undergo post-processing with a simple shortcutting algorithm which sequentially shortens each agent's solution path without introducing conflicts by replacing segments with linear interpolations while avoiding obstacles and other agents. This standard shortcutting method is commonly used to refine paths generated by sampling-based planners \cite{choset2005shortcutting}.

\input{figure_experiments_comparison}

\subsection{Experimental Results}
Table \ref{fig:experimental_incomplete} highlights the mercurial behavior of incomplete constraints when substituting them in ECBS in place of complete constraints. When effective, incomplete constraints like avoidance, step-priority, or sphere can improve success rates, showcasing how these constraints can effectively prune the search space. However, it may not be clear in some domains if constraints are overly pruning. We see that increasing the sphere constraint radius from 5cm (ECBS-S) to 15cm (ECBS-LS) dramatically reduces the success rate in our 8-arm scene as the search space becomes overly constrained. Thus constraints must be used intelligently to provide consistent benefits or risks needing to be re-evaluated for utility in every scenario.

This motivates trying out all the constraints at once with AC-ECBS as described in Section \ref{sec:ac_ecbs}. AC-ECBS in Table \ref{fig:naive_gen_ecbs} demonstrates how naively incorporating all these constraints within the ECBS framework leads to a larger runtime and a lower success rate due to the large branching factor. Using lazy expansions (AC-ECBS-L) improves performance compared to AC-EBCS but marginally against regular ECBS. 
By employing multiple queues and adaptively prioritizing constraint types, Generalized ECBS can effectively capitalize on the benefits of incomplete constraints and achieve a significant improvement in planning time and success rate.

We compare Generalized ECBS with other standard approaches to M-RAMP in Figure \ref{fig:experimental_evaluation} and see that it has a higher success rate than all other baselines.
For 4 and 6 robots, sampling-based methods (PRM, RRT-Connect) impressively succeed in finding solutions to some problems, however, suffer from poor path cost and overall success rate.
With 8 robots, these methods fail, showing how MAPF reasoning (rather than composite space planning) is preferred.
In this challenging scenario, it is particularly interesting to compare Generalized ECBS with Prioritized Planning (PP) which has severe incomplete constraints, and with ECBS which has only complete constraints. We see that Gen-ECBS has a higher success rate than PP, implying that the priority constraints are too strict. However, we see that just complete constraints, ECBS, also performs worse than Generalized ECBS. This highlights how using multiple incomplete and complete constraints is better than either by themselves.

\label{sec:experiments}

%% file: figure_constraints_replacement.tex
\begin{table}[t]
\centering

 \begin{tabular}{|l|c|c|c|c|c|}
        \hline
        8 Robots & \cellcolor[HTML]{C0C0C0}Succ. (\%) & \cellcolor[HTML]{C0C0C0} Runtime (s) & \cellcolor[HTML]{C0C0C0} Cost (rad) \\ \hline
\cellcolor[HTML]{C0C0C0}ECBS & 48\% & 17.7 $\pm$ 19.46 & 38.9 $\pm$ 7.30 \\ \hline
\cellcolor[HTML]{C0C0C0}ECBS-A & 78\% & 13.1 $\pm$ 12.60 & 39.9 $\pm$ 6.64 \\ \hline
\cellcolor[HTML]{C0C0C0}ECBS-P & 78\% & 13.2 $\pm$ 12.57 & 39.9 $\pm$ 6.63 \\ \hline
\cellcolor[HTML]{C0C0C0}ECBS-S & 76\% & 15.8 $\pm$ 15.97 & 40.5 $\pm$ 7.15 \\ \hline
\cellcolor[HTML]{C0C0C0}ECBS-LS & 10\% & 8.6 $\pm$ 10.20 & 30.4 $\pm$ 2.49 \\ \hline
\cellcolor[HTML]{C0C0C0}ECBS-XLS & 4\% & 3.2 $\pm$ 1.78 & 28.5 $\pm$ 1.04 \\ \hline
\end{tabular}

  \caption{An experimental analysis showing the effects of directly replacing the constraints in ECBS with incomplete constraints. Significant improvements in success rate can be attributed to the choice of constraint types. The constraints are Step-Priority (P), Avoidance (A), and sphere with $r=5,15,30\text{cm}$ (S, LS, XLS).}
  \label{fig:experimental_incomplete}
\end{table}

%% file: figure_experiments_all_constraints_naive.tex
\begin{table}
    \centering
        \begin{tabular}{|l|c|c|c|c|c|}
        \hline
        8 Robots & \cellcolor[HTML]{C0C0C0}Succ. (\%) & \cellcolor[HTML]{C0C0C0} Runtime (s) & \cellcolor[HTML]{C0C0C0} Cost (rad) \\ \hline
        \cellcolor[HTML]{C0C0C0}ECBS & 48\% & 17.7 $\pm$ 19.4 & 38.9 $\pm$ 7.3 \\ \hline
        \cellcolor[HTML]{C0C0C0}AC-ECBS & 38\% & 27.4 $\pm$ 17.2 & 38.5 $\pm$ 7.3 \\ \hline
        \cellcolor[HTML]{C0C0C0}AC-ECBS-L & 54\% & 18.1 $\pm$ 17.2 & 40.6 $\pm$ 7.2 \\ \hline
        \cellcolor[HTML]{C0C0C0}Gen-ECBS & \textbf{84\%} & 18.8 $\pm$ 16.3 & 41.2 $\pm$ 6.9 \\ \hline
        \end{tabular}    
    \caption{Naively adding constraints to ECBS (AC-ECBS, Sec. \ref{sec:ac_ecbs}) may hurt performance as more time is spent evaluating child CT nodes. Employing lazy one-step evaluations helps (AC-ECBS-L, Sec. \ref{sec::gen_ecbs}), and prioritizing promising constraint types (Generalized ECBS) improves performance further. The 8-robot setup is described in Section \ref{sec:experiments} and shown in Figure \ref{fig:experimental_evaluation}.}
    \label{fig:naive_gen_ecbs}
\end{table}

%% file: figure_experiments_comparison.tex
\begin{figure*}[t]
\begin{minipage}{0.6\textwidth}
\centering
  \includegraphics[height=3cm]{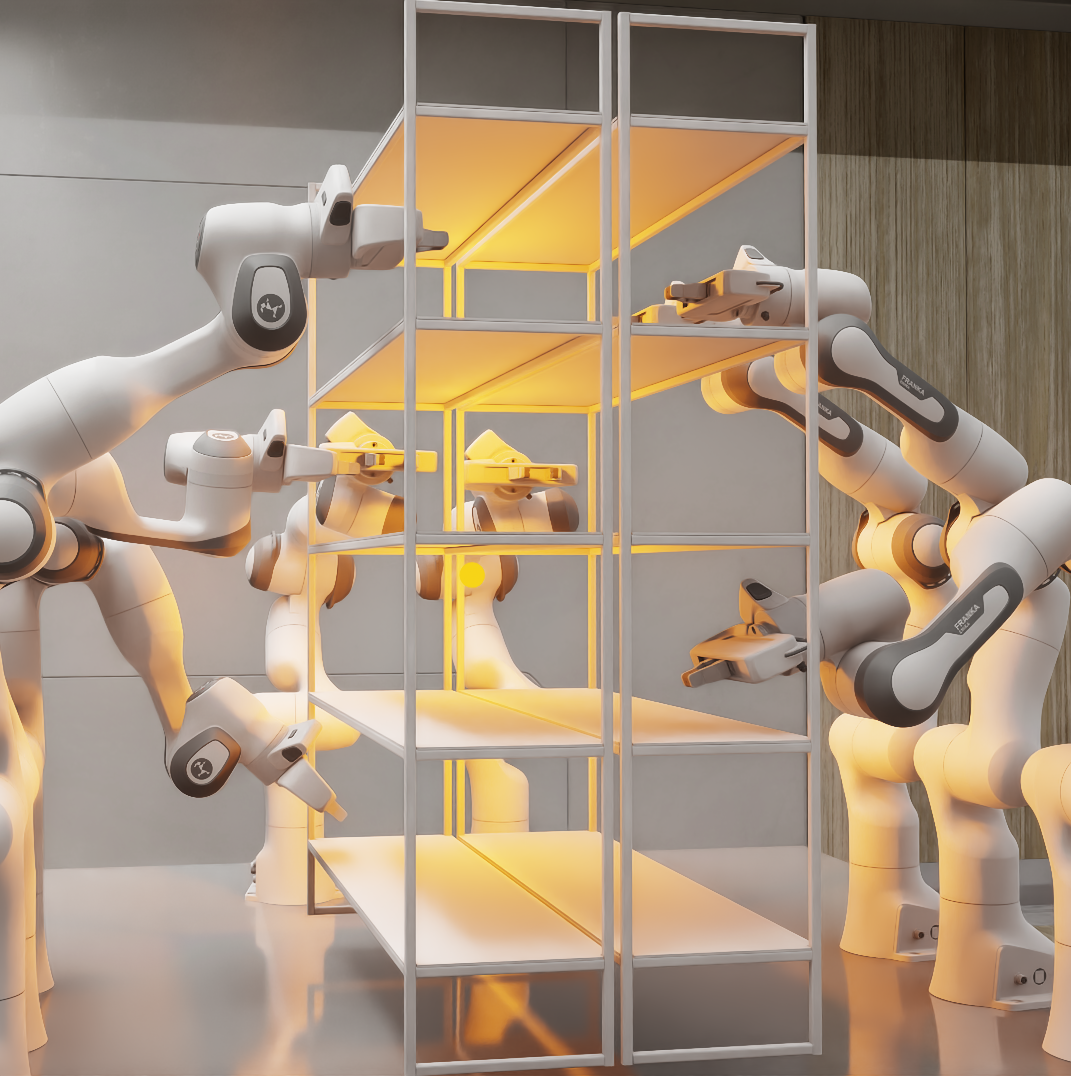}
  \hspace{0.3cm}
  \includegraphics[height=3.05cm]{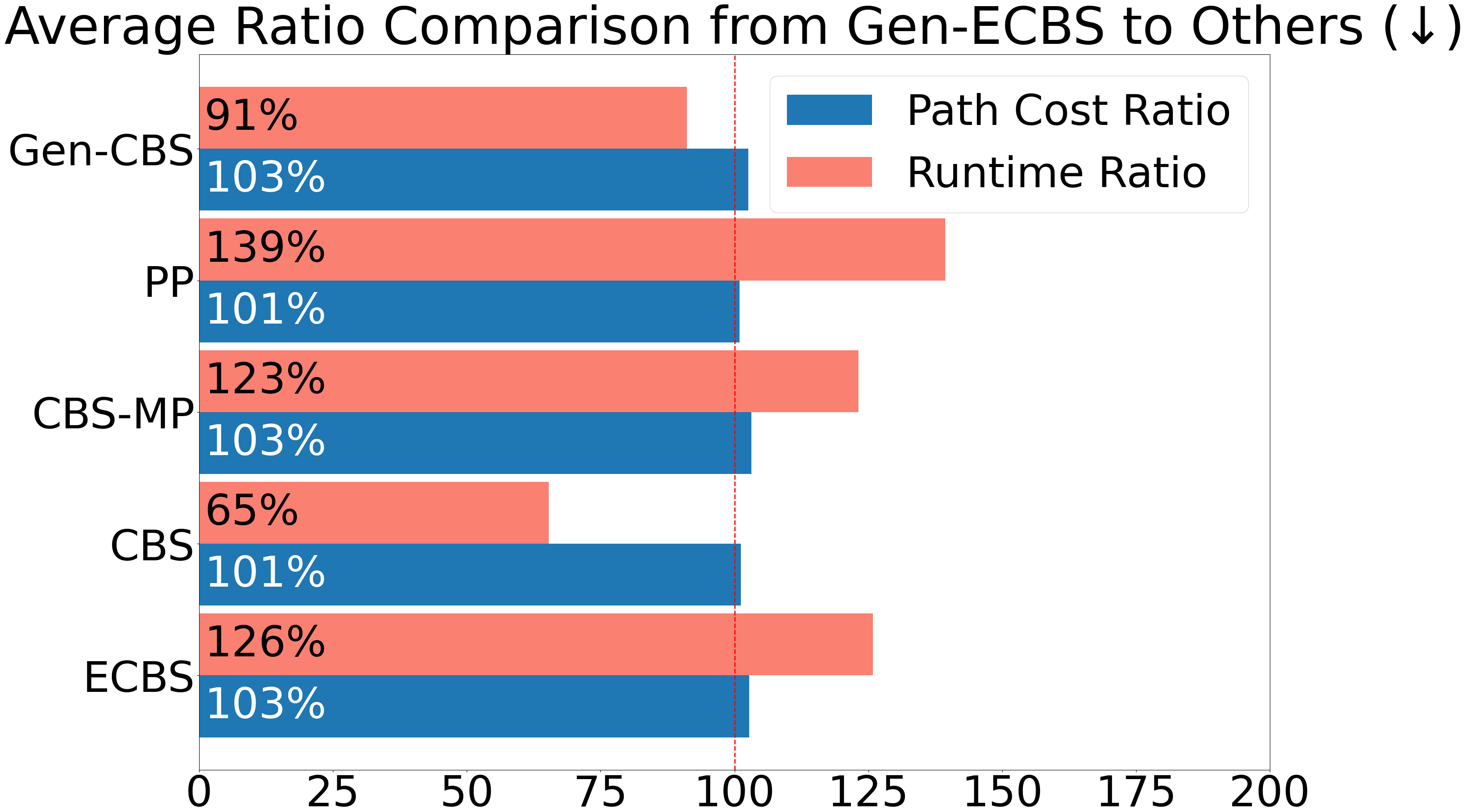}
\end{minipage}
\begin{minipage}{0.4\textwidth}
\resizebox{\textwidth}{!}{%
\hspace{-0.5cm}
\begin{tabular}{|l|c|c|c|c|c|}
\hline
8 Robots & \cellcolor[HTML]{C0C0C0}Success (\%) & \cellcolor[HTML]{C0C0C0} Runtime (sec) & \cellcolor[HTML]{C0C0C0} Cost (rad) \\ \hline
\rowcolor[HTML]{EFEFEF} 
\cellcolor[HTML]{C0C0C0}Gen-ECBS & \textbf{84\%} & 18.8 $\pm$ 16.39 & 41.2 $\pm$ 6.93 \\ \hline
\cellcolor[HTML]{C0C0C0}Gen-CBS & 34\% & 20.9 $\pm$ 17.38 & 36.9 $\pm$ 6.15 \\ \hline
\cellcolor[HTML]{C0C0C0}ECBS & 48\% & 17.7 $\pm$ 19.46 & 38.9 $\pm$ 7.30 \\ \hline
\cellcolor[HTML]{C0C0C0}CBS & 18\% & 23.3 $\pm$ 23.03 & 35.0 $\pm$ 6.08 \\ \hline
\cellcolor[HTML]{C0C0C0}PP & 52\% & 14.5 $\pm$ 9.64 & 39.0 $\pm$ 6.48 \\ \hline
\cellcolor[HTML]{C0C0C0}RRT-Con & 0\% & - & - \\ \hline
\cellcolor[HTML]{C0C0C0}PRM & 0\% & - & - \\ \hline
\cellcolor[HTML]{C0C0C0}CBS-MP & 36\% & 15.2 $\pm$ 15.19 & 36.9 $\pm$ 5.93 \\ \hline
\end{tabular}
}
\end{minipage}
\hspace{0.2cm}
\begin{minipage}{0.6\textwidth}
\centering
      \includegraphics[height=3cm]{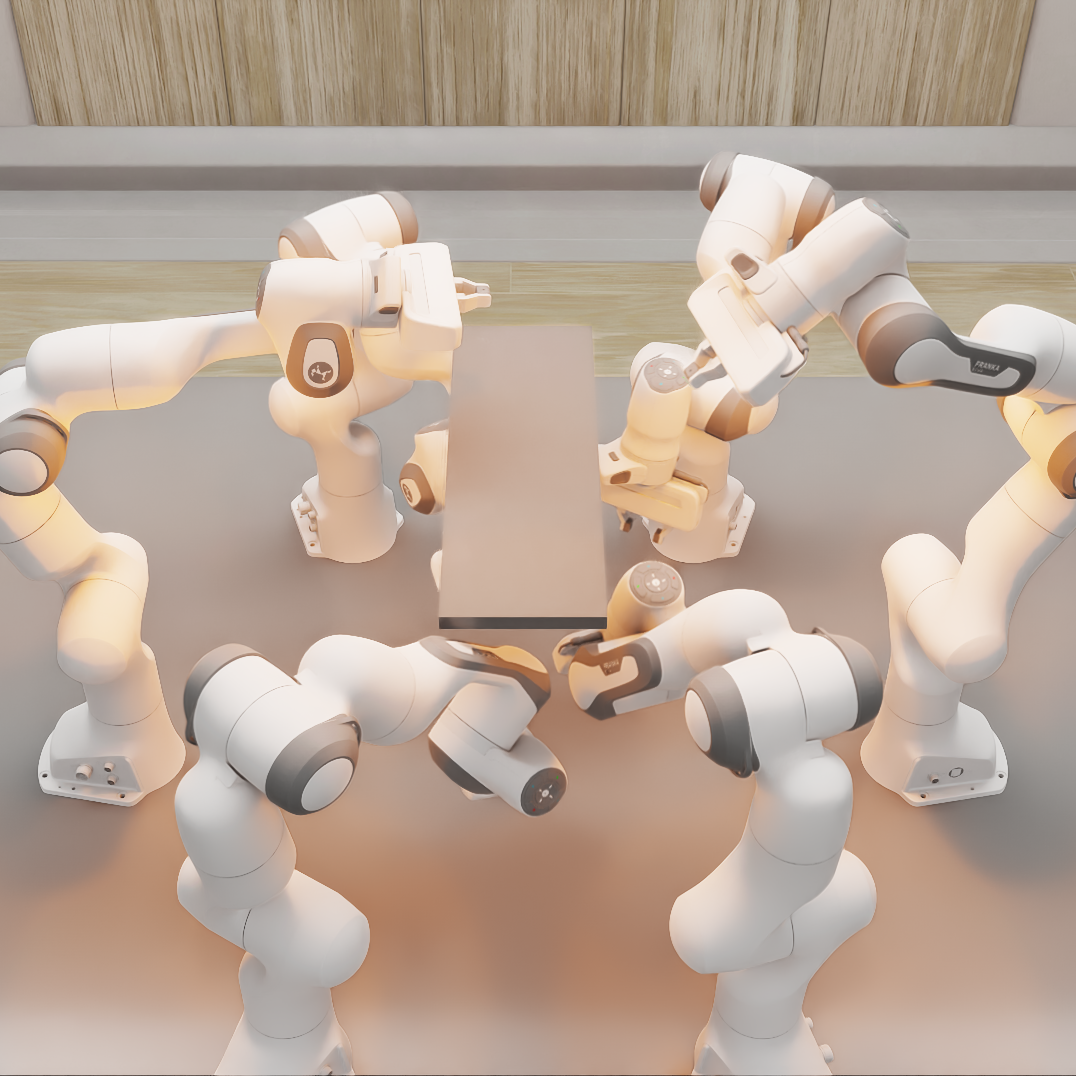}
        \hspace{0.3cm}
      \includegraphics[height=3.05cm]{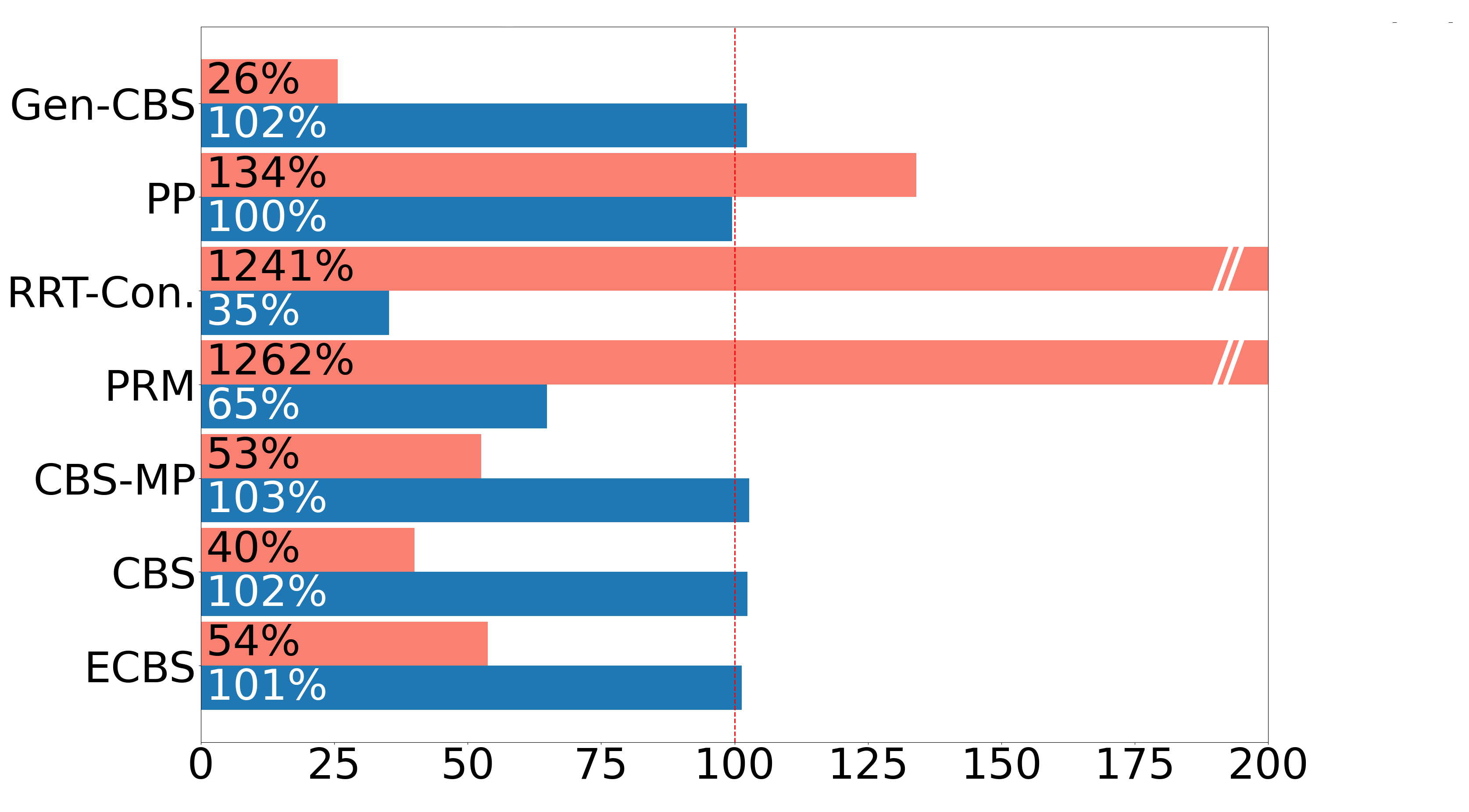}
    \end{minipage}
\begin{minipage}{0.4\textwidth}
\resizebox{\textwidth}{!}{%
\hspace{-0.5cm}
\begin{tabular}{|l|c|c|c|c|c|}
\hline
6 Robots & \cellcolor[HTML]{C0C0C0}Success (\%) & \cellcolor[HTML]{C0C0C0} Runtime (sec) & \cellcolor[HTML]{C0C0C0} Cost (rad) \\ \hline
\rowcolor[HTML]{EFEFEF} 
\cellcolor[HTML]{C0C0C0}Gen-ECBS & \textbf{84\%} & 7.4 $\pm$ 11.43 & 31.8 $\pm$ 5.01 \\ \hline
\cellcolor[HTML]{C0C0C0}Gen-CBS & 42\% & 15.3 $\pm$ 15.19 & 29.4 $\pm$ 4.77 \\ \hline
\cellcolor[HTML]{C0C0C0}ECBS & 70\% & 11.6 $\pm$ 10.26 & 31.2 $\pm$ 5.26 \\ \hline
\cellcolor[HTML]{C0C0C0}CBS & 22\% & 14.1 $\pm$ 18.59 & 29.9 $\pm$ 5.12 \\ \hline
\cellcolor[HTML]{C0C0C0}PP & 72\% & 12.7 $\pm$ 14.41 & 31.5 $\pm$ 4.68 \\ \hline
\cellcolor[HTML]{C0C0C0}RRT-Con. & \textbf{84}\% & 1.4 $\pm$ 1.71 & 103.9 $\pm$ 35 \\ \hline
\cellcolor[HTML]{C0C0C0}PRM & 18\% & 17.2 $\pm$ 20.61 & 131.6 $\pm$ 108 \\ \hline
\cellcolor[HTML]{C0C0C0}CBS-MP & 54\% & 15.9 $\pm$ 18.44 & 30.3 $\pm$ 5.16 \\ \hline
\end{tabular}
}
\end{minipage}
\hspace{0.2cm}
\begin{minipage}{0.6\textwidth}
\centering
      \includegraphics[height=3cm]{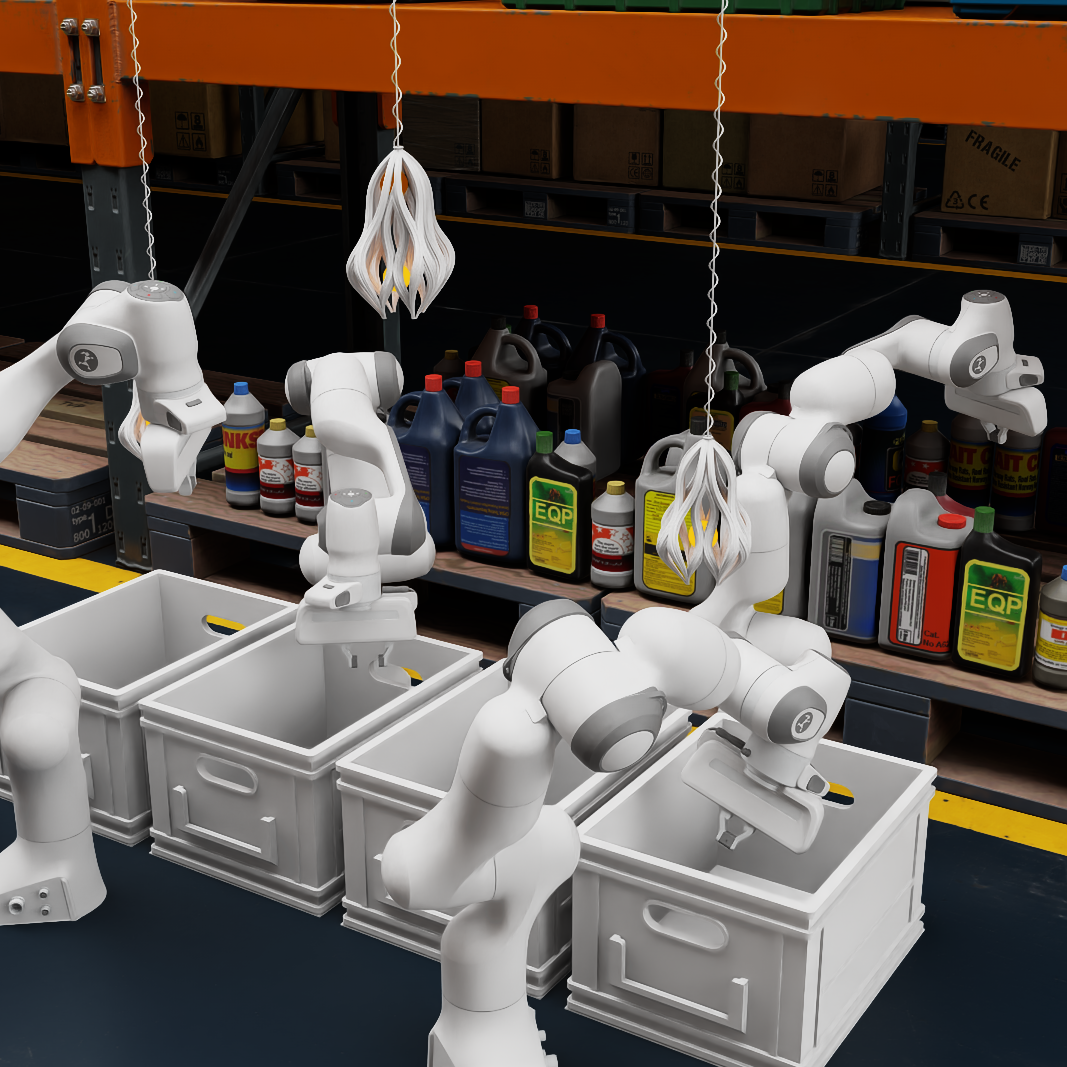}
        \hspace{0.3cm}
      \includegraphics[height=3.05cm]{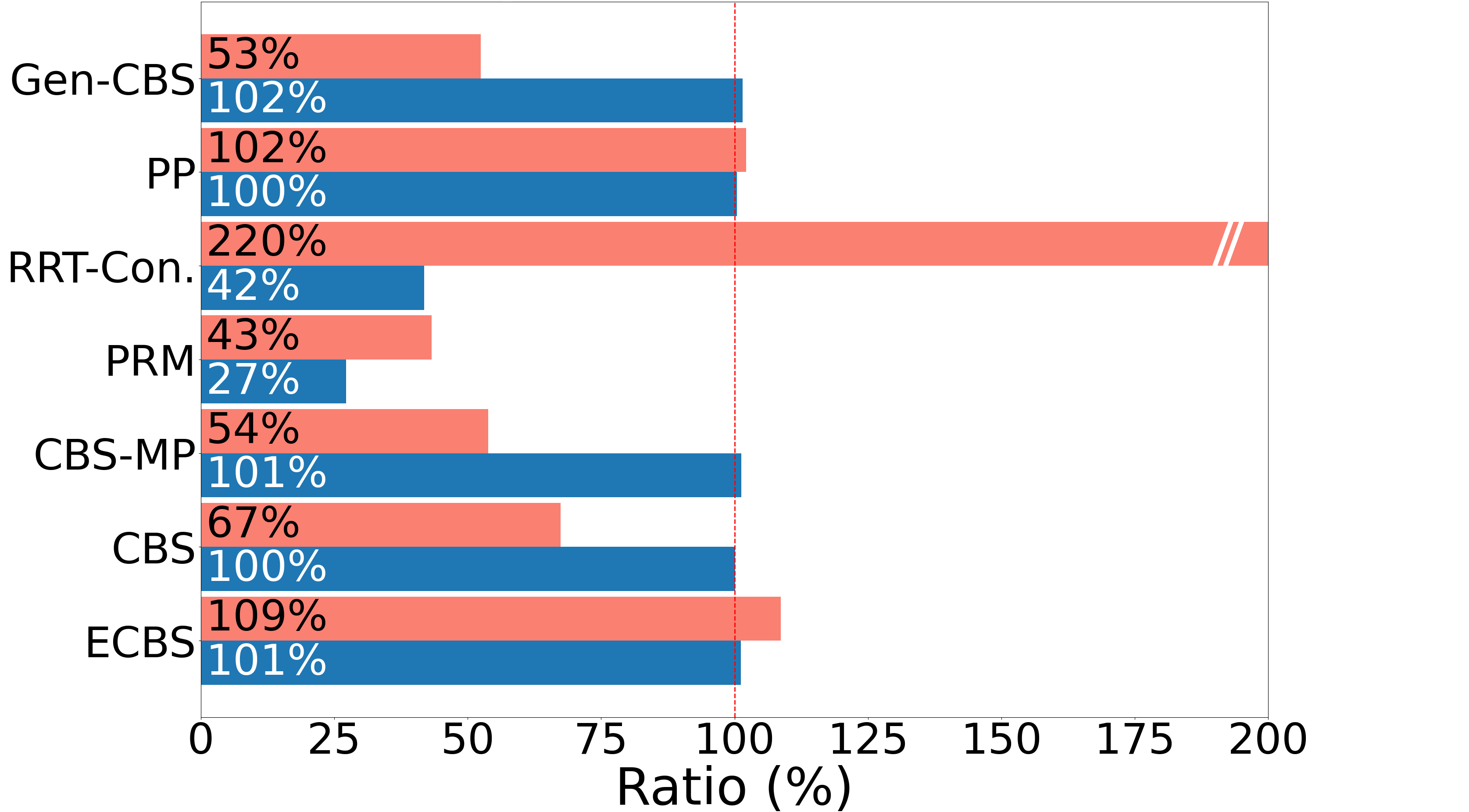}
\end{minipage}
\begin{minipage}{0.4\textwidth}
\resizebox{\textwidth}{!}{%
\hspace{-0.5cm}
\begin{tabular}{|l|c|c|c|c|c|}
\hline
4 Robots & \cellcolor[HTML]{C0C0C0}Success (\%) & \cellcolor[HTML]{C0C0C0} Runtime (sec) & \cellcolor[HTML]{C0C0C0} Cost (rad) \\ \hline
\rowcolor[HTML]{EFEFEF} 
\cellcolor[HTML]{C0C0C0}Gen-ECBS & \textbf{98\%} & 9.5 $\pm$ 11.10 & 23.2 $\pm$ 3.66 \\ \hline
\cellcolor[HTML]{C0C0C0}Gen-CBS & 60\% & 15.3 $\pm$ 15.57 & 21.4 $\pm$ 2.91 \\ \hline
\cellcolor[HTML]{C0C0C0}ECBS & 86\% & 12.5 $\pm$ 15.08 & 22.6 $\pm$ 3.61 \\ \hline
\cellcolor[HTML]{C0C0C0}CBS & 28\% & 9.1 $\pm$ 9.61 & 20.4 $\pm$ 2.89 \\ \hline
\cellcolor[HTML]{C0C0C0}PP & 88\% & 11.4 $\pm$ 13.98 & 22.9 $\pm$ 3.51 \\ \hline
\cellcolor[HTML]{C0C0C0}RRT-Con. & 32\% & 10.9 $\pm$ 10.00 & 60.5 $\pm$ 19.55 \\ \hline
\cellcolor[HTML]{C0C0C0}PRM & 2\% & 13.3 $\pm$ 0.0 & 78.9 $\pm$ 0.0 \\ \hline
\cellcolor[HTML]{C0C0C0}CBS-MP & 68\% & 13.9 $\pm$ 14.06 & 21.7 $\pm$ 2.99 \\ \hline
\end{tabular}
}
\end{minipage}
  \caption{Comparing planning algorithms in realistic M-RAMP problems. {Left:} Renders of the planning scenes. Two scenes, with 4 and 8 robots, exhibit dense obstacle clutter while the scene with 6 robots is more open. {Right:} success rate, planning time, and cost results across 50 tests ($\mu \pm \sigma$). Through a high success rate, we see that Generalized ECBS scales with the number of robots and handles clutter. {Middle:} a pairwise comparison between Generalized ECBS and the other methods. Looking at tests where both methods succeeded, we report the average ratio of Generalized ECBS's path cost and planning time to that of the other. The path cost is nearly identical among search-based planners. Since Generalized ECBS solves more problems than the others, the pairwise comparison focuses on simpler tests, where Generalized ECBS planning time may show a slight overhead.}
  \label{fig:experimental_evaluation}
\end{figure*}

%% file: conclusions.tex
\section{Conclusion}
Existing MAPF works are ubiquitous in 2D grid worlds, but may struggle in more realistic domains like Multi-Robot-Arm Motion Planning (M-RAMP). A key factor in their success is the type of constraints they use.
We observed that in current frameworks, methods either apply complete constraints and are theoretically complete and bounded sub-optimal but slow, or apply stronger constraints and are fast but incomplete. Practitioners in today's paradigm are thus forced to pick between slow methods with guarantees or fast methods that can fail on solvable instances. 

We bridge this gap with Generalized-ECBS, which enables using arbitrary incomplete constraints without giving up completeness or bounded sub-optimality. Gen-ECBS uses lazy expansions, multiple focal queues, and dynamic prioritization to effectively use multiple constraints for a higher and more consistent success rate.
In our experiments, we observed that Generalized-ECBS is effective in high-dimensional M-RAMP with a few simple constraints. Our proposed approach is domain-agnostic, and 
we are excited about the possibilities that it opens to future research in MAPF and multi-arm manipulation.